\documentclass[final]{cvpr}
\usepackage{tablefootnote}
\usepackage{times}
\usepackage{epsfig}
\usepackage{graphicx}
\usepackage{amsmath}
\usepackage{amssymb}
\usepackage{xcolor}
\usepackage[ruled,vlined]{algorithm2e}
\usepackage{booktabs}
\usepackage{amsthm}
\usepackage{cool}

\SetKwInput{KwInput}{Input}                %
\SetKwInput{KwOutput}{Output}              %

\usepackage[pagebackref=true,breaklinks=true,colorlinks,bookmarks=false]{hyperref}
\def\cvprPaperID{7686} %
\def\confYear{CVPR 2021}
\DeclareMathOperator*{\argmax}{arg\,max}
\DeclareMathOperator*{\argmin}{arg\,min}

\newif\ifdraft
\drafttrue
\draftfalse

\begin{document}

\renewcommand{\vec}[1]{\mathbf{#1}}
\newcommand{\myrightarrow}[1]{\xrightarrow{\makebox[2em][c]{$\scriptstyle#1$}}}
\newtheorem{hypothesis}{Hypothesis}
\newtheorem{lemma}{Lemma}

\newcommand{\pder}[2][]{\frac{\partial#1}{\partial#2}}
\newcommand\norm[1]{\left\lVert#1\right\rVert}

\title{\vspace*{-0.5in}{{\normalsize \rm In the 2021 IEEE/CVF Conference on Computer Vision and Pattern Recognition\hrule}}\vspace*{0.4in}Data-Free Model Extraction}

\author{Jean-Baptiste Truong\thanks{equal contribution}\\
Worcester Polytechnic Institute\\
{\tt\small jtruong2@wpi.edu}
\and
Pratyush Maini\footnotemark[1]~\thanks{Work done while an intern at the Vector Institute.}\\
Indian Institute of Technology Delhi\\
{\tt\small pratyush.maini@gmail.com}
\and
Robert J. Walls\\
Worcester Polytechnic Institute\\
{\tt\small rjwalls@wpi.edu}
\and
Nicolas Papernot\\
University of Toronto and Vector Institute\\
{\tt\small nicolas.papernot@utoronto.ca}
}

\maketitle

\begin{abstract}

Current model extraction attacks assume that the adversary has access to a surrogate dataset with characteristics similar to the proprietary data used to train the victim model. This requirement precludes the use of existing model extraction techniques on valuable models, such as those trained on rare or hard to acquire datasets. In contrast, we propose data-free model extraction methods that do not require a surrogate dataset. 
Our approach adapts techniques from the area of data-free knowledge transfer for model extraction. As part of our study, we identify that the choice of loss is critical to ensuring that the extracted model is an accurate replica of the victim model.  Furthermore, we address difficulties arising from the adversary's limited access to the victim model in a black-box setting. For example, we recover the model's logits from its probability predictions to approximate gradients. We find that the proposed data-free model extraction approach achieves high-accuracy with reasonable query complexity -- 0.99$\times$ and 0.92$\times$ the victim model accuracy on SVHN and CIFAR-10 datasets given 2M and 20M queries respectively.

\end{abstract}

\section{Introduction}

Machine learning (ML) and deep learning, in particular, often require large amounts of training data to achieve high performance on a particular task~\cite{strubell2019energy}. 
Curating such data necessitates significant time and monetary investment~\cite{data2009,imagenet2009}. Thus, the resulting ML model becomes valuable intellectual property, especially when considering the computing resources and human expertise required~\cite{brown2020language,dosovitskiy2020image}. 
Often to monetize these models,  companies make them available as a service via APIs over the web (MLaaS). 
These models are also deployed to end-user devices, making their predictions directly accessible to customers.  
However, the exposure of the model's predictions represents a significant risk as an adversary can leverage this information to steal the model's knowledge~\cite{lowd2005learning,tramer2016stealing,chandrasekaran2019exploring,pal2019framework,Orekondy_2019_CVPR,Correia_Silva_2018,milli2018model,jagielski2020high}. 
The threat of such model extraction attacks is two-fold: adversaries may use the stolen model for monetary gains or as a reconnaissance step to mount further attacks~\cite{papernot2017recon, shumailov2020sponge}.

While model extraction is in many ways similar to model distillation, it differs in that  the victim's proprietary training set is not accessible to the adversary. To stage a model extraction attack, the adversary typically queries the victim using  samples from a surrogate dataset with semantic or distributional similarity to the original training set~\cite{Orekondy_2019_CVPR}.  In the classification setting, the victim's response may be limited to the most-likely label~\cite{Chandrasekaran2018ModelEA} 
or include  confidence values for different class labels~\cite{jagielski2020high}.
The number of queries---i.e., the \emph{query complexity}---is also an important consideration for the adversary. The greater the query complexity, the higher the cost of the attack---unless the victim model is available offline (e.g., deployed on-device). %

In this work, we first demonstrate in Section~\ref{sec:surrogate} that the success of current established practices for model extraction, which often take the form of distillation, depends on the \textit{closeness} of the surrogate distribution to the victim's proprietary training distribution. %
This finding has important implications for the practicality of existing model extraction techniques.

To remedy this, we propose techniques for \emph{data-free model extraction} (DFME). %
In short, we demonstrate the feasibility of extracting ML models without \emph{any} knowledge of the distribution of the proprietary training data. In practice, gathering a surrogate dataset for the purpose of model extraction can be a very expensive process, both in terms of the time and money required to curate it. In particular, the most valuable models are often those for which it is most challenging to curate an appropriate surrogate dataset, i.e., when the victim model's value arises from its proprietary dataset.
Our work builds on recent advances in data-free knowledge distillation, which involve a generative model to synthesize queries that maximize disagreement between the student and teacher models~\cite{micaelli2019zero,fang2019datafree}. Here, the teacher is the victim model whereas the student is the stolen extracted model. We innovate on two fronts: the choice of loss to quantify student-teacher disagreement and an approach for training the generator without the ability to backpropagate through the teacher to compute its gradients (because we only have black-box access to the victim/teacher predictions in our setting).
We observe that it is essential to ensure the stability of the loss computed, and find that the $\ell_1$ norm loss is particularly conducive to data-free model extraction. 
We also demonstrate that using inexpensive gradient approximation (based on the victim  model's outputs) is sufficient to  train a generative model that produces  queries relevant to distill the knowledge of a victim to a student model.
In summary, our main contributions are:
\begin{itemize}
\itemsep0em 
\item We demonstrate in Section~\ref{sec:surrogate} that successful distillation-based model extraction attacks require the adversary to
sample queries from a surrogate dataset
whose distribution is \textit{close} to the victim training data.
\item In Section~\ref{sec:methods}, we propose data-free model extraction (DFME)  to extract ML models without knowledge of  private training data, and only using the victim's black-box predictions. As a by-product of DFME needing to approximate gradients of the victim, this leads us to present a method for recovering per-example logits out of the  probability vector output by a ML model.

\item We validate\footnote{Code and models for reproducing our work can be found at \href{https://github.com/cake-lab/datafree-model-extraction}{https://github.com/cake-lab/datafree-model-extraction}} our DFME technique in Section~\ref{sec:exp} on the SVHN and CIFAR10 datasets and successfully extract a model with 0.99x the victim accuracy with only 2M queries for SVHN, and 0.92x the victim accuracy with 20M queries for CIFAR10.

\item An ablation study of our approach in Section~\ref{sec:ablation} provides two key insights: (1)  measuring disagreement between the victim and extracted models with the $\ell_1$ norm achieves higher extraction accuracy than losses previously considered in the literature; (2) weak gradient estimates yield sufficient signal to train a generator despite only having access to the victim's predictions.

\end{itemize}

\section{Related Work}
\label{sec:related}

We covered the seminal results in model extraction based on surrogate datasets in the introduction.  Here, we discuss data-free knowledge distillation---the technique that underlies our approach to data-free model extraction---as well as the rudiments of generative modeling and gradient approximation required to understand our method.

\subsection{Data-Free Knowledge Distillation}

Knowledge distillation aims to compress, i.e., transfer, the knowledge of a (larger) teacher model to a (smaller) student model~\cite{ba2014deep,hinton2015distilling}. It was originally introduced to reduce the size of models deployed on devices with limited computational resources.
Since then, this line of work has attracted a lot of attention~\cite{zhang2017deep,gou2020knowledge,romero2015fitnets,zagoruyko2017paying, zhang2019teacher}.
While the model owner usually performs knowledge distillation, the original dataset used to train the teacher model may not be available during distillation~\cite{micaelli2019zero}, e.g., because the dataset is too large or confidential. 
Therefore, others have proposed distillation techniques that leverage a surrogate dataset with a similar feature space or distribution~\cite{lopes2017datafree,Orekondy_2019_CVPR}. Others proposed techniques that altogether remove the need for a surrogate dataset, i.e., data-free knowledge distillation~\cite{fang2019datafree,micaelli2019zero}.
 Techniques addressing data-free knowledge distillation have relied on training a generative model to synthesize the queries that the student makes to the teacher~\cite{choi2020data, micaelli2019zero}.

The success of data-free knowledge distillation hints at the feasibility of data-free model extraction. 
Kariyappa et al. observe this as well in concurrent work~\cite{kariyappa2020maze}. 
They also
tackle data-free model extraction through the synthesis of queries by a generative model. Key differences include our loss formulation and optimizer choice (see Section~\ref{sec:methods}).
We show in Sections~\ref{sec:exp} and~\ref{sec:ablation} that our approach consistently outperforms theirs.

\subsection{Generative Models}

Model extraction through data-free distillation involves the generation of training data with which the \emph{student} (i.e., adversary) queries the \emph{teacher} (i.e., victim) model. 
Naively, one could generate these queries randomly~\cite{micaelli2019zero, fang2019datafree}. 
In this paper, we instead build on a min-max game between two adversaries that try to optimize opposite loss functions. 
This approach is analogous to the optimization performed in Generative Adversarial Networks (GANs)~\cite{goodfellow2014generative} to train the generator and discriminator. 
Here, we use GANs in a fashion analogous to their application to semi-supervised learning~\cite{salimans2016improved}: our student and teacher models, in conjunction, play the discriminator's role. 
The key difference here is that GANs are generally trained to recover an underlying \textit{fixed} data distribution. 
However, our generator chases a moving target: the distribution of data which is most indicative of the discrepancies between the decision surfaces of the current student model and its teacher model.

\subsection{Black-box Gradient Approximation}

Zeroth-order optimization is a common approach to approximating gradients~\cite{wang2018stochastic, nesterov2017random,chen2017zoo,liu2020primer}. Such techniques have previously been used to mount attacks against ML models in a \textit{black-box} setting, e.g., to craft
adversarial examples~\cite{tu2019autozoom, chen2017zoo, bhagoji2018practical}. Various gradient estimation methods solve different trade-offs between query complexity and the quality of the gradient estimate~\cite{tu2019autozoom, chen2017zoo, bhagoji2018practical}. We use the forward differences~\cite{wibisono2012finite} 
method for its relatively low query utilization, and systematically study the impact of its main parameter (e.g. the number of random directions) in Section~\ref{subsec:grad-approx}.

\section{How Hard is it to Find a Surrogate Dataset?}
\label{sec:surrogate}
\begin{figure}[t]
\centering
\includegraphics[height=0.23\textwidth]{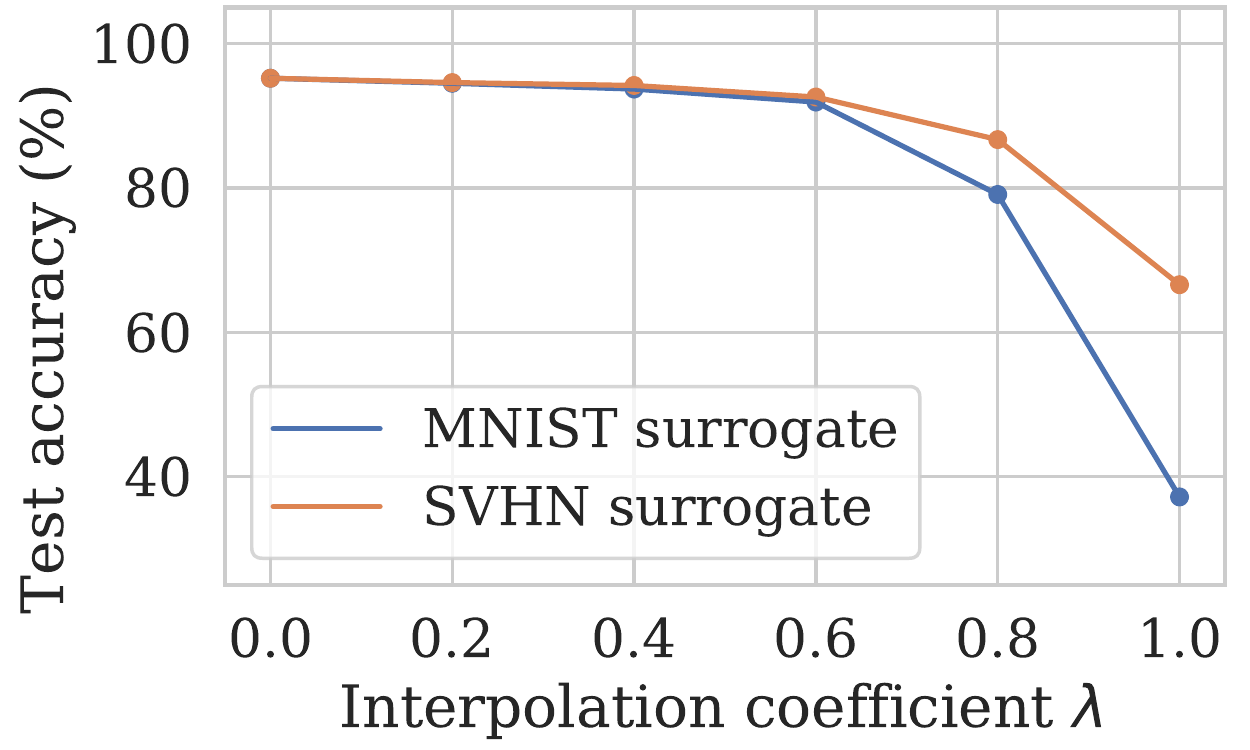}
\caption{Dataset Interpolation with CIFAR10 as target. $\lambda=0$ implies that the inputs are sampled from the target distribution, while $\lambda=1$ implies sampling from the surrogate.
\label{fig:surrogate}}
\end{figure}

To motivate the need for data-free approaches to model extraction,  we evaluate if an adversary must ensure that the distribution of its surrogate dataset is \textit{close} to that of the victim's training dataset. We hypothesize that in the absence of this condition, distillation-based model extraction will return a poor approximation of the victim. 
We perform an analysis on the \textit{closeness} of the distributions along three axes: (1) similarity in feature space, (2) marginal probability distribution of inputs, and (3) class-conditional probability distribution of the inputs. In our experiments, we attempt to steal ML models trained on CIFAR10~\cite{krizhevsky2009learning} and SVHN~\cite{netzer2011reading} using various surrogate datasets that align differently with the axes defined above.  We study in details the experimental setting, optimization problem, surrogate datasets and hyperparameters in Appendix~\ref{app:surrogate}.

Our experiments support our hypothesis. For instance, in case of CIFAR10, with a victim model of accuracy 95.5\%, extracting it using CIFAR100~\cite{krizhevsky2009learning} as surrogate dataset results in extraction accuracy of 93.5\%. This can be largely attributed to the fact that both the CIFAR10 and CIFAR100 datasets are subsets from the same TinyImages~\cite{torralba2008tinyimages} dataset.
However, on using SVHN as surrogate dataset, the model extraction performance dropped remarkably, attaining a maximum of 66.6\% across all the hyperparameters tried. 
In the extreme scenario when querying the CIFAR10 teacher with MNIST~\cite{lecun1998gradient}--a dataset with disjoint feature space both in terms of number of pixels, and number of channels)--- accuracy did not improve beyond 37.2\%.

On the contrary, we notice that the victim trained on the SVHN dataset is much easier for the adversary to extract. Surprisingly, even when the victim is queried with completely random inputs, the extracted model attains an accuracy of over 84\% on the original SVHN test set. 
We hypothesize that this observation is linked to how the digit classification task, at the root of SVHN, is a simpler task for neural networks to solve, and the underlying representations (hence, not being as complex as for CIFAR10) can be learnt even when queried over random inputs.

While these correlations agree with our hypothesis, these experiments can not systematically quantify the \emph{distance} between two distributions (viz. the surrogate and the target). To more systematically understand how the shift away from the target distribution affects extraction performance we interpolated inputs ($x_{in}$) from the surrogate ($x_s$) and target ($x_t$) datasets, s.t. $x_{in} = (1-\lambda)\cdot x_{t} + \lambda\cdot x_{s}$. 
Figure~\ref{fig:surrogate} shows the decrease in extraction accuracy as the distribution diverges from target (CIFAR10) for two different surrogate datasets (SVHN and MNIST).

We make two conclusions from our observations: (1) the success of distillation-based model extraction largely depends on the complexity of the task that the victim model aims to solve; and (2) similarity to source domain appears to be critical for extracting ML models that solve complex tasks. 
We posit that it may be nearly as expensive for the adversary to extract such models with a good surrogate dataset, as is training from scratch. A weaker or non-task specific dataset may have lesser costs, but has high accuracy trade-offs.

\section{Data-Free Model Extraction}
\label{sec:methods}

The goal of model extraction  is to train a student model $\mathcal{S}$ to match the predictions of the victim $\mathcal{V}$ on its private target domain $\mathcal{D_V}$. That is to say, find the student model's parameters $\theta_S$ that minimize the probability of errors between the student and victim predictions $\mathcal{S}(x)$ and $\mathcal{V}(x)~\forall x\in\mathcal{D_V}$:
\begin{equation}
    \argmin_{\theta_S} ~\mathcal{P}_{x\sim {\mathcal{D_V}}} \left(
    \argmax_{i}
    \mathcal{V}_i(x) \neq \argmax_{i} \mathcal{S}_i(x)\right)
\end{equation}
Since the victim's domain, $\mathcal{D_V}$, is not publicly available, the proposed data-free model extraction attack minimizes the student's error on a synthesized dataset, $\mathcal{D}_S$. The error is minimized by optimizing a loss function, $\mathcal{L}$, which measures disagreement between the victim and student:
\begin{equation}
    \argmin_{\theta_S} ~\mathbb{E}_{x\sim {\mathcal{D}_S}} \left[\mathcal{L} (\mathcal{V}(x), \mathcal{S}(x))\right]
\end{equation}
This section describes how we minimize the number of queries made to the victim model with a novel  query generation process, and how we train the student model itself. 

\subsection{Overview}

The overall attack setup is inspired by Generative Adversarial Networks~\cite{goodfellow2014generative}. A generator ($\mathcal{G}$) model is responsible for crafting some input images, and the student model $\mathcal{S}$  serves as a discriminator while trained to match the victim $\mathcal{V}$  predictions on these images. In this setting, the two adversaries are $\mathcal{S}$ and $\mathcal{G}$, which respectively try to minimize and maximize the disagreement between $\mathcal{S}$ and $\mathcal{V}$.

The data flow is shown as a black arrow in Figure~\ref{fig:dfme-diagram}: a vector of random noise $\vec{z}$ is sampled from a standard normal distribution and fed into $\mathcal{G}$ which produces an image $\vec{x}$. Then the victim $\mathcal{V}$ and student $\mathcal{S}$ each perform inference on $\vec{x}$ to finally compute the loss function $\mathcal{L}$. 

During the back-propagation phase (shown with red arrows) gradients from two different sources need to be computed: the gradients of $\mathcal{L}$ with regards to the student's parameters $\theta_S$ and the gradient of $\mathcal{L}$ with regards to the generator's parameters $\theta_G$. Because the victim is only accessible as a black-box, it is not possible to propagate gradients through it. The dashed arrow indicates the need for gradient approximation (see Section~\ref{subsec:grad-approx}). %

\begin{figure} 
 \centering
  \includegraphics[width=\columnwidth]{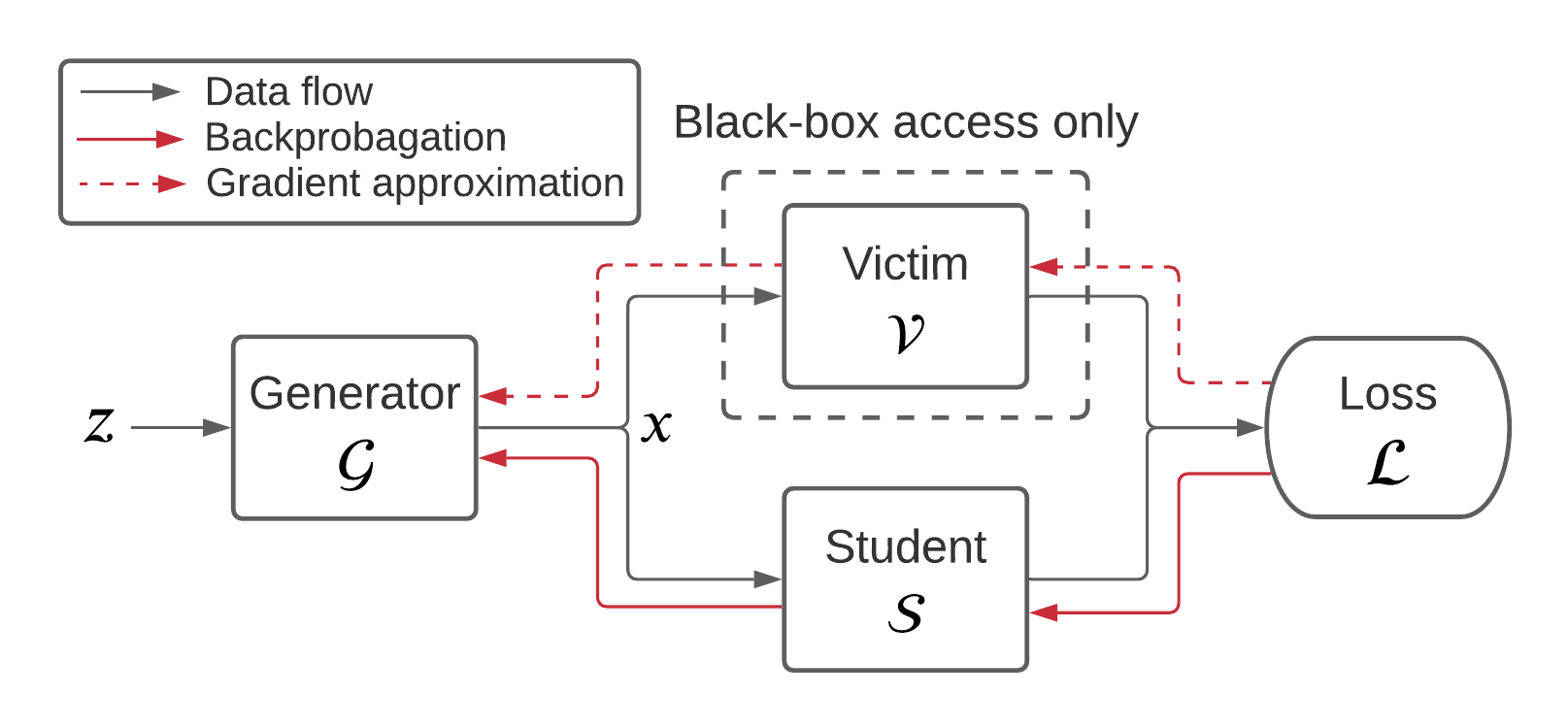}
  \caption{Date-Free Model Extraction Attack Diagram}
  \label{fig:dfme-diagram}
\end{figure}

\paragraph{Student.}
Prior work on knowledge distillation showed that a student model $\mathcal{S}$ can learn from a teacher and reach high accuracy even though its architecture is smaller and different~\cite{cho2019efficacy,micaelli2019zero}. Therefore, in the context of model extraction, the adversary only needs to select a model architecture which has sufficient capacity. This does not require knowledge of the victim architecture but rather generic knowledge of architectural choices made for the task solved by the victim (e.g., a convolutional neural network is appropriate for an object recognition task).  
In our work, we used a student with ResNet-18-8x architecture for model extraction.

The loss function $\mathcal{L}$ is used to measure the disagreement between $\mathcal{S}$ and $\mathcal{V}$. For this function, we use the $\ell_1$ norm loss between victim and student logits (i.e. pre-softmax activations), $l_i(x)$ and $s_i(x)$ respectively. This requires us to recover the logits from the softmax outputs, since the adversary only has access to the later. We introduce an approach for doing so and further elaborate on the choice of $\mathcal{L}$ is detailed in Subsection~\ref{subsec:loss-function}. It is important to note that the gradient of the loss with regard to the student's weights $\theta_S$ does not require gradients of $\mathcal{V}$ since the victim's predictions don't depend on the weights $\theta_S$. 

\paragraph{Generator.}
The generator model $\mathcal{G}$ is used to synthesize images that maximize the disagreement between $\mathcal{S}$ and $\mathcal{V}$. The loss function used for $\mathcal{G}$ is the same as for $\mathcal{S}$ except that the goal is to maximize it. From this setting emerges an adversarial game in which $\mathcal{S}$ and $\mathcal{G}$ compete to respectively maximize and minimize the same function. In other words, the student is trained to match the victim's predictions and the generator is trained to generate difficult examples for the student. The adversarial game can be written as:
\begin{equation}
    \underset{\mathcal{S}}{\min} ~\underset{\mathcal{G}}{\max} ~\mathbb{E}_{z\sim {\mathcal{N}(0, 1)}} \left[\mathcal{L} (\mathcal{V}(\mathcal{G}(z)), \mathcal{S}(\mathcal{G}(z)))\right]
\end{equation}
As shown in Figure~\ref{fig:dfme-diagram}, computing the gradient of $\mathcal{L}$ with regard to $\theta_G$ requires gradients of $\mathcal{V}$. As we only have access to $\mathcal{V}$ as a black-box, gradient approximation techniques are required. These techniques are discussed in Section~\ref{subsec:grad-approx}.

\paragraph{Algorithm.}
Each iteration alternates training the generator $\mathcal{G}$ and  student $\mathcal{S}$. To finely tune the balance between $\mathcal{G}$ and $\mathcal{S}$ training, each of these training phases is  repeated $n_G$ and $n_S$ times, respectively, before moving on to the next epoch. While setting $n_G$ higher allows $\mathcal{G}$ to train faster and to produce more difficult examples for $\mathcal{S}$, it can also be wasteful if $\mathcal{S}$ does not see enough examples. The trade-off between $n_G$ and $n_S$ is an additional hyperparameter that needs tuning. The additional hyperparameters $m$ and $\epsilon$ are related to gradient approximation (see Section~\ref{subsec:grad-approx}).

\begin{algorithm}[t]
\SetAlgoLined
\KwInput{Query budget $Q$, generator iterations $n_\mathcal{G}$, student iterations $n_\mathcal{S}$, learning rate $\eta$, random directions $m$, step size $\epsilon$}
\KwResult{Trained $\mathcal{S}$}
 \While{$Q > 0$}{
  \For{$i = 1 \dots n_G$} {
    $\vec{z} \sim \mathcal{N}(0,\,1)$\\
    $x = \mathcal{G}(z; \theta_\mathcal{G} )$\\
    approximate gradient $\nabla_{\theta_\mathcal{G}} \mathcal{L}(x) $\\
    $\theta_\mathcal{G} = \theta_\mathcal{G} - \eta \nabla_{\theta_\mathcal{G}}\mathcal{L}(x)$\\
  }
  \For{$i = 1\dots n_S$} {
    $z \sim \mathcal{N}(0,\,1)$\\
    $x = G(z; \theta_\mathcal{G} )$\\
    compute $\mathcal{V}(x)$, $\mathcal{S}(x)$, $\mathcal{L}(x)$, $\nabla_{\theta_\mathcal{S}} \mathcal{L}(x)$\\
    $\theta_\mathcal{S} = \theta_\mathcal{S} - \eta \nabla_{\theta_\mathcal{S}}\mathcal{L}(x)$
  }
  update remaining query budget $Q$
 }
 \caption{Data-Free Model Extraction}
\end{algorithm}

\subsection{Loss function}
\label{subsec:loss-function}

Here we discuss different loss functions to measure the disagreement between $\mathcal{V}$ and $\mathcal{S}$.  These losses are commonly used in the knowledge distillation literature given the similarity with the model extraction task~\cite{cho2019efficacy,fang2019datafree}. The choice of the loss function is key to the outcome of the attack since gradients computed through $\mathcal{S}$ and $\mathcal{V}$ can easily impede the convergence of optimizers, e.g., if they vanish because the wrong loss function is used. 

\paragraph{Kullback–Leibler (KL) Divergence} Most prior work in model distillation optimized over the KL divergence between the student and the teacher~\cite{cho2019efficacy,hinton2015distilling,lan2018knowledge}. As a result, KL divergence between the outputs of $\mathcal{S}$ and $\mathcal{V}$ is a natural candidate for the loss function to train the student network.
For a probability distribution over $K$ classes indexed by $i$, the KL divergence loss for a single image $x$ is defined as:
\begin{equation}
    \mathcal{L}_{\text{KL}}(x) = \sum_{i=1}^K \mathcal{V}_i(x) \log \left(\frac{\mathcal{V}_i(x)}{\mathcal{S}_i(x)}\right)
\end{equation}
However, as the student model matches more closely the victim model, the KL divergence loss tends to suffer from vanishing gradients~\cite{fang2019datafree}. Hypothesis~\ref{th:grad-vanish} suggests that $\mathcal{L}_{\text{KL}}$ can make it difficult to achieve convergence while training $\mathcal{G}$ (refer to Appendix~\ref{app:proof-th} for justification).
Specifically, back-propagating such vanishing gradients through the generator can harm its learning. We confirm this  through empirical evaluation as well in  Section~\ref{sec:exp-loss}.

\begin{hypothesis}
\label{th:grad-vanish}
The gradients of the KL divergence loss with respect to the image $x$ should be small compared to the gradients of the $\ell_1$ norm loss when $\mathcal{S}$ converges to $\mathcal{V}$: 
\begin{equation*}
    \|\nabla_x \mathcal{L}_{\text{KL}}(x)\| \underset{\mathcal{S} \rightarrow \mathcal{V}}{\ll} \|\nabla_x \mathcal{L}_{\ell_1}(x)\| 
\end{equation*}
\end{hypothesis}
\paragraph{The $\ell_1$ norm loss.}
To prevent gradients from vanishing, we use the  $\ell_1$ norm loss ($\mathcal{L}_{\ell_1}$) computed with the victim and student logits $v_i$ and $s_i$ where $i \in \{1...K\}$ and $K$ is the number of classes. This was previously found by Fang et al. to prevent gradients from vanishing in knowledge distillation~\cite{fang2019datafree}. Even though $\mathcal{L}_{\ell_1}$ is not differentiable everywhere, it does not suffer from the vanishing gradients issue and yields better results in practice (see Sec.~\ref{sec:exp-loss}). Lastly, the probabilities output by $\mathcal{V}$ need to be transformed into logits to be used in $\mathcal{L}_{\ell_1}$. 
We describe how to perform logit approximation in Appendix~\ref{app:logits-correction}, and evaluate in Sec.~\ref{subsec:exp-logits}.
\begin{equation}
    \mathcal{L}_{\ell_1}(x) = \sum_{i=1}^K|v_i - s_i|
\end{equation}

\subsection{Gradient Approximation}
\label{subsec:grad-approx}

Because only black-box access is provided for $\mathcal{V}$, the optimizer aims at maximizing a function for which it only has an evaluation oracle.
Yet, in order to train $\mathcal{G}$, gradients of the loss with regards to $\mathcal{G}$'s parameters $\nabla_{\theta_G} \mathcal{L}$ must be computed. Thus, we approximate gradients by interacting with the oracle: we maximize $\mathcal{L}$ with zeroth-order optimization.

\subsubsection{Images as a Proxy}

The number of parameters in $\mathcal{G}$ is typically large (millions of parameters) and it would be very query-expensive for a zeroth-order optimizer to get accurate gradient estimations on this large space. Instead, one can approximate gradients with regards to the input images $x$, and then back-propagate this gradient through $\mathcal{G}$~\cite{kariyappa2020maze}. This way the dimensionality of gradients being approximated   is much smaller, which yields more accurate zeroth-order approximations.

Additionally the oracle might only accept images that lie within a pre-defined input domain, for example $[-1 ,1]^d$. To force $\mathcal{G}$ to respect this constraint, we use a hyperbolic tangent activation at the end of the generator architecture. Furthermore, zeroth-order gradients approximation methods usually evaluate the function in the neighborhood of a given point, which can result in query images slightly outside the input domain. To avoid this, we approximate gradients with regard to the pre-activation images (i.e. just before the hyperbolic tangent function is applied).

\subsubsection{Forward Differences Method}
\label{subsubsec:existing-grad-approx}

The Forward Differences method approximates gradients by computing directional derivatives $\mathcal{D}_{\vec{u}_i} f(x)$ of a function $f$ at a point $x$ along $m$ random directions $\vec{u}_i$. The directional derivatives are computed by measuring the variation of $f$ a small step of size $\epsilon$ in the direction $\vec{u}_i$. They are then averaged to form an estimator of the gradient $\nabla_{\text{FWD}} f(x)$. In a way, each directional derivative brings some amount of information about true gradient. The estimator being more accurate as the number of random directions increases. 
\begin{equation}
\label{eq:ForwardD}
    \nabla_{\text{FWD}} f(x) = \frac{1}{m} \sum_{i=1}^m d \frac{f(x + \epsilon \vec{u_i}) - f(x)}{\epsilon} \vec{u_i}
\end{equation}
The main advantage of this method is that the number of query directions $m$ may be chosen independently  of the input space dimensionality, offering a trade-off between query utilization and gradient accuracy. This makes it an appealing candidate for DFME~\cite{kariyappa2020maze}. The influence of the number of query directions $m$ is further described in Section~\ref{sec:res-grad-approx}.

Finite differences, an alternative gradient approximation method used when crafting adversarial examples~\cite{bhagoji2018practical}, requires too many queries per gradient estimate to be viable for data-free model extraction.

\section{Experimental Validation}
\label{sec:exp}

We evaluate data-free model extraction (DFME) against 
victim models trained SVHN and CIFAR-10. We show that the resulting student models can reach high accuracy (e.g., 95.2\% on SVHN) even when the generator only has access to inaccurate gradient estimates. 
Later in Section~\ref{sec:ablation}, we perform an ablation study and evaluate the impact of each attack component on the final student model accuracy and on the query budget Q.

\subsection{Datasets and Architectures}

We evaluate the effectiveness of the proposed DFME method on two datasets: SVHN and CIFAR-10. For each dataset, the victim model architecture is a ResNet-34-8x. These victim models were trained during 50 epochs for SVHN and 200 for CIFAR-10 with SGD at an initial learning rate of 0.1, decayed by a factor of 10 at 50\% of training.

We use ResNet-18-8x as the architecture for our student model. This is inspired by previous works in knowledge distillation~\cite{fang2019datafree} that show how a smaller student is sufficient to distill the knowledge of a larger teacher.
The network was trained with a batch size of 256 with SGD, with an initial learning rate of $0.1$, a weight-decay of $5.10^{-4}$, and a learning rate scheduler that multiplies the learning rate by a factor 0.3 at 0.1$\times$, 0.3$\times$, and 0.5$\times$ the total training epochs. The default query budget $Q$ is 2M for SVHN, and 20M for CIFAR-10 in our experiments. 

The generator used three convolutional layers, interleaved with linear up-sampling layers, batch normalization layers, and ReLU activations for all layers except the last one. The final activation function was the hyperbolic tangent function to output values in the range [-1,1] (see Section~\ref{subsec:grad-approx}). It was also trained with a batch size of 256, but using an Adam optimizer with an initial learning rate of $5.10^{-4}$ which is decayed by a factor 0.3 at 10\%, 30\%, and 50\% of the training. 

For gradient approximation we sample $m=1$ random directions and a step size $\epsilon=10^{-3}$. 

\subsection{Results}
\label{subsec:full-evaluation}

We compare the performance of different extraction attacks in Table~\ref{tab:overall-accuracy}. 
We measure the ratio between the student's accuracy and the victim's accuracy on the victim's test set. This helps compare the performance of DFME across different datasets. The student model's normalized accuracy is reported 
for each dataset and extraction method evaluated---our approach (DFME), our approach with KL divergence loss (DFME-KL), our approach without logit correction (Log-Probabilities), and concurrent work~\cite{kariyappa2020maze} (MAZE). Further, we perform DFME with a range of query budgets and reported the accuracy in Figure~\ref{fig:acc-vs-budget}.

Without any knowledge of the original training distribution, the proposed DFME method achieved as high as 88.1\% (0.92x target) of accuracy with Q = 20M and 89.9\% (0.94x target) with Q = 30M. 
The accuracy of the extracted model exceeds that reported in concurrent work which we refer to as MAZE in our results~\cite{kariyappa2020maze}. However, in our best-efforts at reproducing their results with the details in the paper,\footnote{The authors declined to share their code upon request.}  we were unsuccessful in achieving the same reported accuracy, and were only able to achieve an accuracy of  45.6\% (0.48x target) at best on the student model.
In addition, MAZE reports that they were unable to learn when using extremely few directions (such as $m=1$) for the gradient approximation with CIFAR10, whereas we find that weak gradient approximations are beneficial to reduce the overall query budget of successful attacks.

We observed similar results for the SVHN victim model: reaching as high as 95.2\% (0.99x target) accuracy with only 2M queries. The task for SVHN is much simpler than CIFAR10 given that a model with 84\% (0.87x target) accuracy can be extracted from just random noise. The proposed method allows one to achieve far higher accuracy.

One limitation of our study is that the reported query budgets do not include the cost of hyperparameter tuning. This is an important direction for future work as preliminary experiments suggest that extraction accuracy can be sensitive to the choice of hyperparameters.

\begin{table*}[]
\centering
\begin{tabular}{@{}lrrrrr@{}}
\toprule
Dataset (budget)   & Victim accuracy & DFME & DFME-KL   & MAZE*~\cite{kariyappa2020maze} & Log-Probabilities \\ \midrule
CIFAR10 (20M) & 95.5\%            & 88.1\% \footnotesize{(0.92$\times$)} & 76.7\% \footnotesize{(0.80$\times$)} & 45.6\% \footnotesize{(0.48$\times$)} & 73.2\% \footnotesize{(0.77$\times$)}               \\
SVHN (2M)     & 96.2\%            & 95.2\%  \footnotesize{(0.99$\times$)} & 84.7\% \footnotesize{(0.88$\times$)} & 91.1\% \footnotesize{(0.95$\times$)} & 94.4\% \footnotesize{(0.98$\times$)}               \\ \bottomrule
\end{tabular}
\caption{Accuracy and normalized accuracy of data-free model extraction methods. Results for `MAZE' reflect our best-effort reproduction.}
\label{tab:overall-accuracy}
\end{table*}

\begin{figure}
 \centering
  \includegraphics[width=\columnwidth]{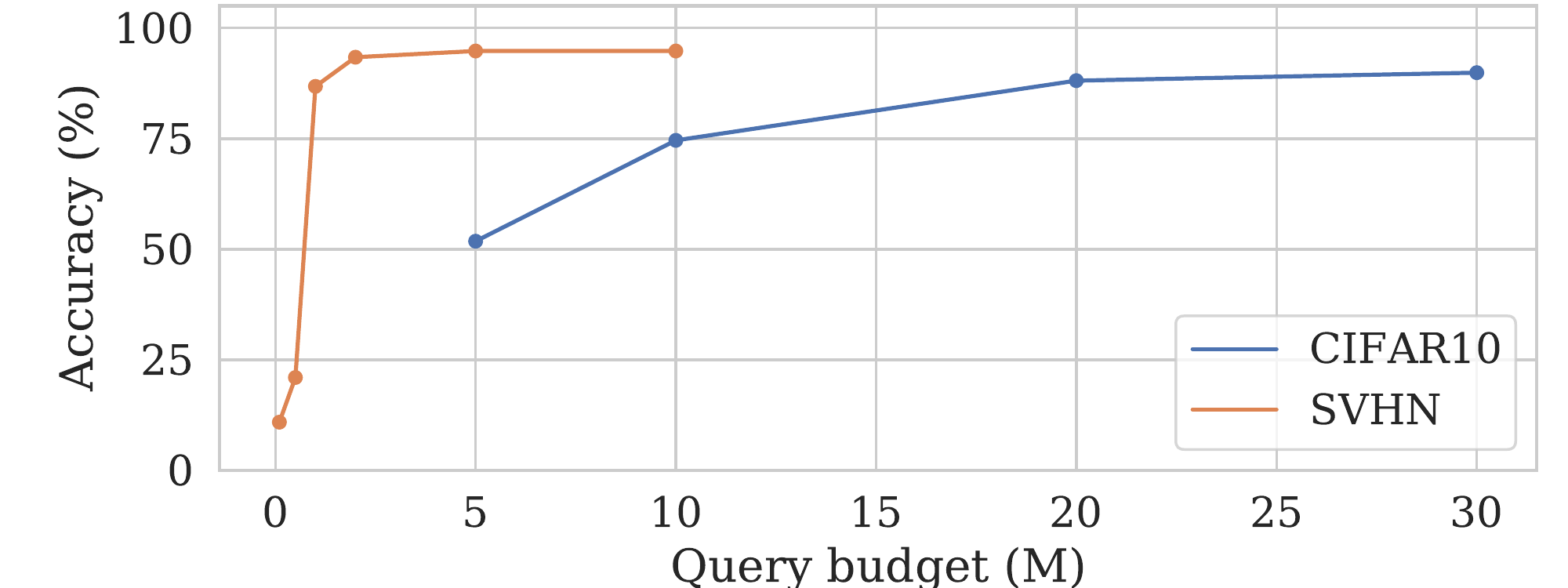}
  \caption{Test accuracy wrt query budget, for SVHN and CIFAR10}
  \label{fig:acc-vs-budget} 
\end{figure}

\section{Ablation Studies}
\label{sec:ablation}
Our work systematically transitions from a data-free knowledge distillation paradigm \cite{fang2019datafree,micaelli2019zero} to a data-free model extraction scenario. The main challenges in this transition were (1) to surpass the need for true gradients for training the student; (2) the lack of access to true victim logits; and (3) the need to restrict the query complexity of the attacks (to reduce the cost of stealing). With this goal, we made specific choices with regards to (a) the loss function; (b) gradient approximation; and (c) logit access. In this section, we detail the impact of each of these choices to the final performance of our proposed DFME method.

\subsection{Choice of Loss Function}
\label{sec:exp-loss}

The choice of loss is of paramount importance to a successful extraction. 
In our DFME approach, the choice of loss involves similar factors to those outlined in research on GANs: multiple works have discussed the problem of vanishing gradients as the discriminator becomes strong in case of GAN training~\cite{arjovsky2017principled,arjovsky2017wasserstein}. For our DFME approach, we minimize the $\ell_1$ distance between the output logits of the student and the teacher. We find that this significantly improves convergence and stability over other possible losses, such as the KL divergence chosen in~\cite{kariyappa2020maze}.

We perform DFME in the same setting to evaluate the difference between the KL divergence and $\ell_1$ losses. Below, we draw comparisons based on two metrics: (1) Final accuracy attained by the student at the end of a fixed number of queries as well as the learning curves of the student; and (2) The norm of gradients of the loss with respect to the input image as the training progresses. 

\paragraph{Test Accuracy.} The key metric of interest for this comparison is the normalized accuracy of the student model at the end of a designated query budget Q of 20M queries for CIFAR10 and 2M queries for SVHN. Table~\ref{tab:overall-accuracy} shows that using the $\ell_1$ loss achieves significantly better test accuracies compared to the KL divergence loss. For instance, on CIFAR10 the accuracy improves from 76.7\% to 88.1\% when switching from KL divergence to the $\ell_1$ loss. 
We also visualize a learning curve for CIFAR10 in  Figure~\ref{fig:acc-vs-epoch}: the KL divergence objective slows converge and tappers off earlier, even when the student has yet to plateau.

\begin{figure}
 \centering
  \includegraphics[width=\columnwidth]{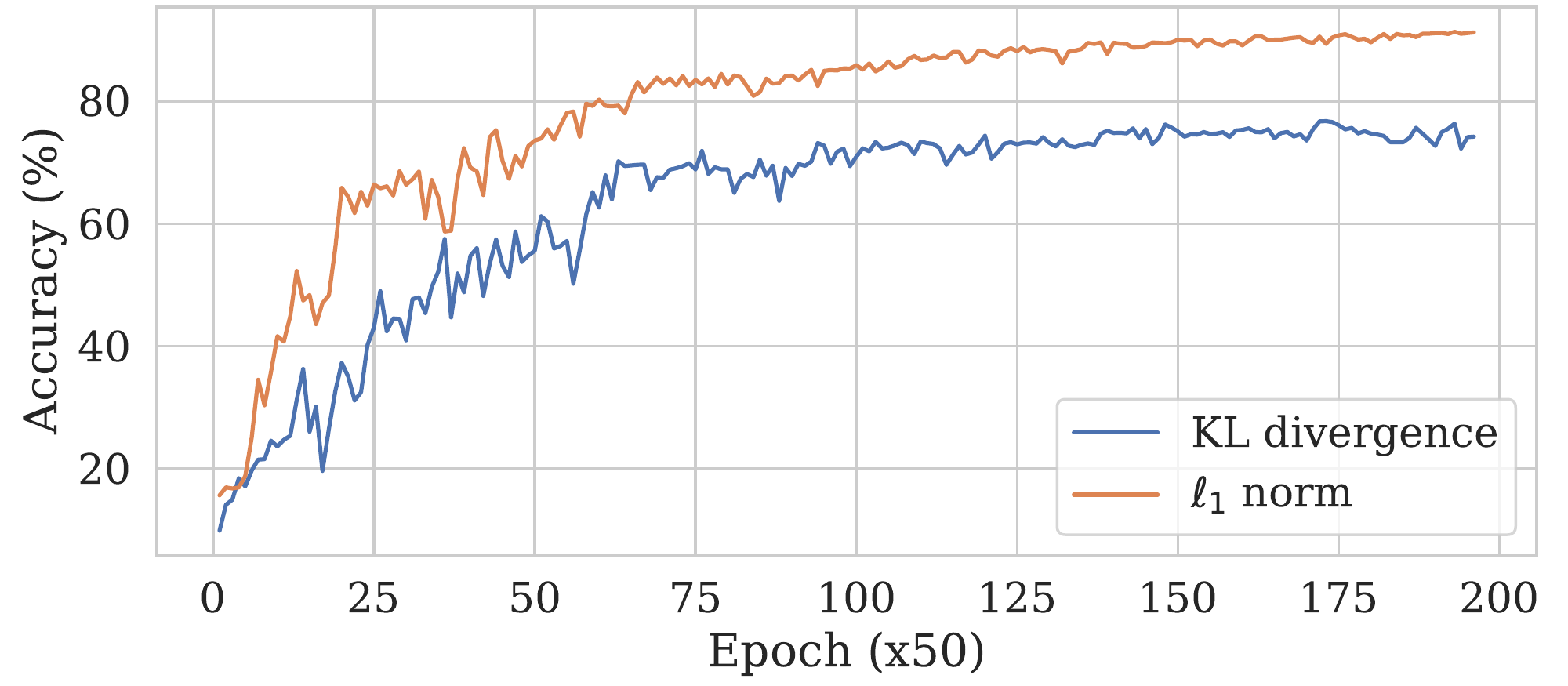}
  \caption{Test accuracy as training progresses for $\ell_1$ and KL divergence losses.}
  \label{fig:acc-vs-epoch} 
\end{figure}

\paragraph{Gradient Vanishing.}The KL divergence loss suffers from vanishing gradients, as explained in Section~\ref{subsec:loss-function}. In DFME, these gradients are used to update the generator's parameters and are thus essential to synthesize queries which extract more information from the victim. In Figure~\ref{fig:vanishing-gradients} we empirically demonstrate that as the student accuracy approaches that of the victim model, the gradients of the KL divergence loss with respect to the input image reduce significantly in norm.
The same decay is slower and less significant in case of the $\ell_1$ loss. We hypothesize that these vanishing gradients are the cause for degraded accuracy when using the KL divergence loss.

\begin{figure}
 \centering
  \includegraphics[width=\columnwidth]{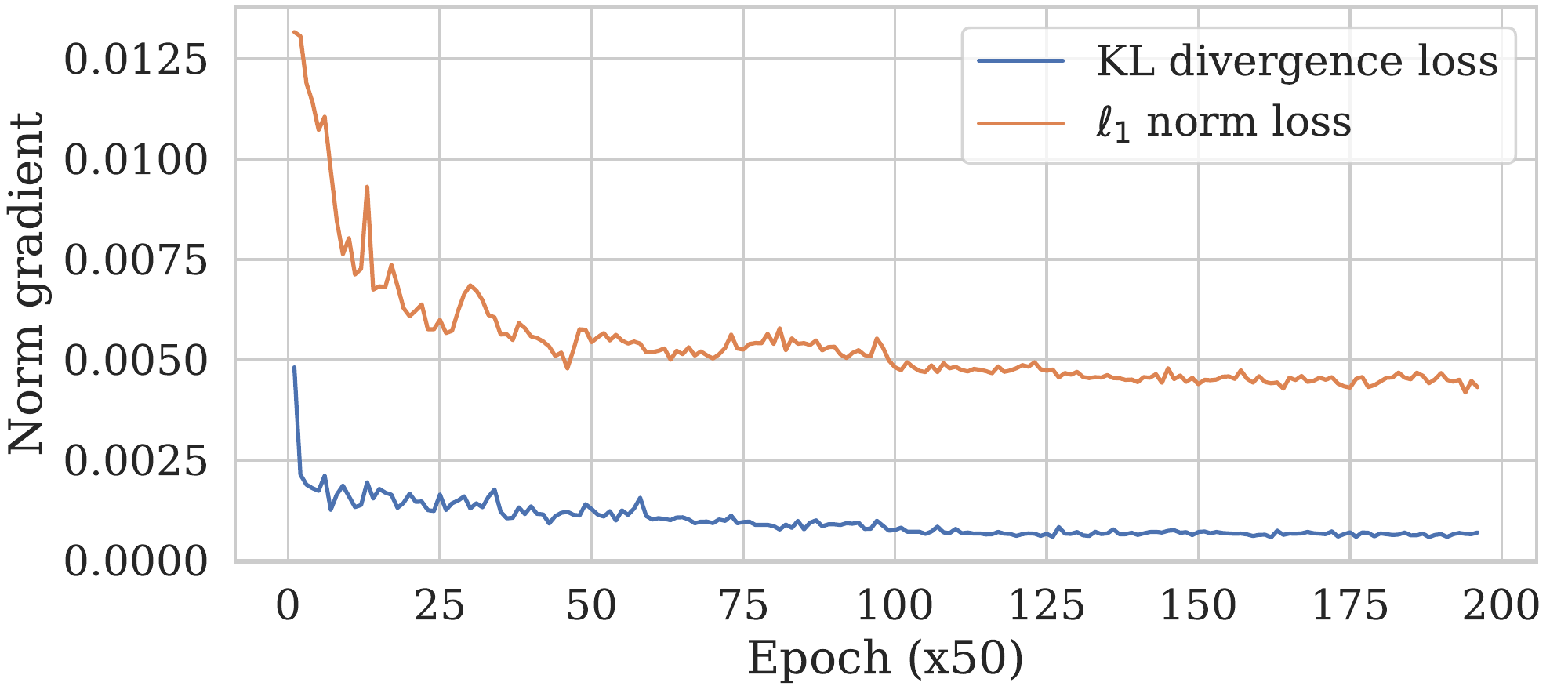}
  \caption{Norm of gradients with respect to the input image, for the KL divergence and $\ell_1$ norm losses.}
  \label{fig:vanishing-gradients} 
\end{figure}

\begin{table}[t]
\centering
\begin{tabular}{@{}l|rrrrr@{}}
\toprule
$m$                 & 1  & 3  & 5  & 8  & 10 \\ \midrule
No. of Queries & 10.04 & 10.02 & 16.33 & 13.80 & 20.00 \\ \bottomrule
\end{tabular}
\caption{Minimum queries (in millions) to reach 85\% accuracy on CIFAR10, for different number of gradient approximation steps.}
\label{table:m-values}
\end{table}

\begin{figure}
 \centering
  \includegraphics[width=\columnwidth]{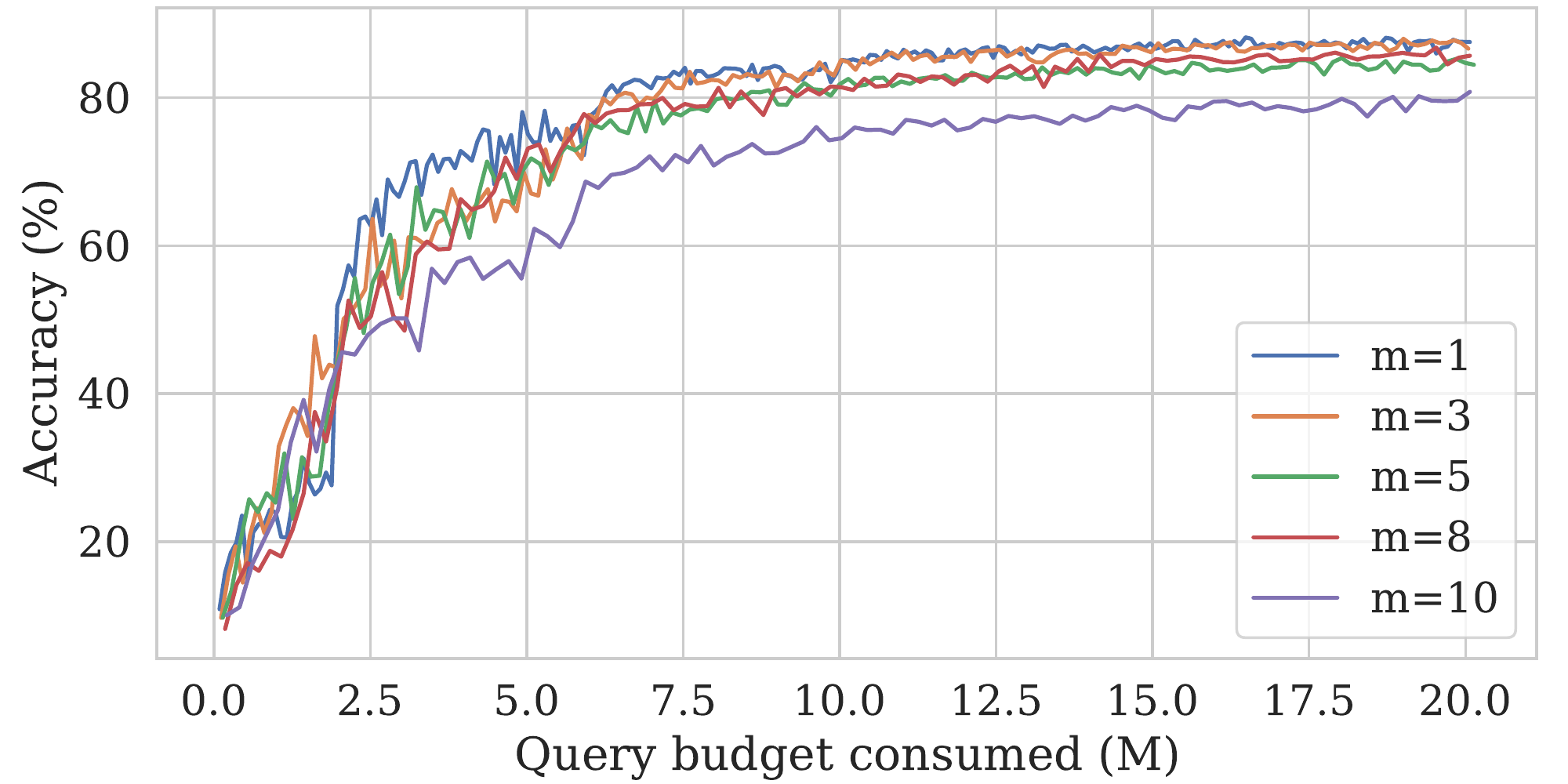}
  \caption{Accuracy of the model during training for different number of gradient approximation steps, $m$. }
  \label{fig:amortizing} 
\end{figure}

\begin{table}[t]
\centering
\scalebox{0.9}{
\begin{tabular}{@{}lrlrrl@{}}
\toprule
        Model&                    \multicolumn{1}{r}{MTL \small(±2.3e-6)
        }                     &  & \multicolumn{1}{r}{MC \small(±2.3e-6)
        } & \multicolumn{1}{r}{LP \small(±1.3)
        }& \\
\midrule
Resnet-34-8x  & -1.24e-5                     &  & 1.24e-5                   & 4.98      &                     \\
Densenet 121 & 5.53e-7                      &  & 1.88e-6                   & 3.88       &                    \\
VGG 16        & \multicolumn{1}{r}{1.97e-5} &  & 1.97e-5                   & 3.93         &                  \\ \bottomrule
\end{tabular}}

\caption{Mean of true logits (MTL) for 3 victim architectures; reconstruction error (MAE) between approximate and true logits when using mean correction (MC) and log-probabilities (LP). }
\label{table:logits-correction}
\end{table}

\subsection{Gradient Approximation}
\label{sec:res-grad-approx}
Recall that 
the $\ell_1$ loss cannot be back-propagated 
through the victim since the adversary only has access to it as a black-box. Recent work on data-free model distillation~\cite{fang2019datafree} has claimed that the gradient information from the teacher is `indispensable at the beginning of adversarial training' because the student alone can not provide useful signal to the generator when randomly initialized.
Below we consider: 
(1) the quality of approximation required; and 
(2) the overall impact on query complexity.
In particular, we compare two approaches for improving the training of our DFME generator: using an increased number of queries to compute more accurate gradient estimates or training generator for longer using poorer gradient estimates. %

\paragraph{Number of gradient approximation steps.} When gradients used to update the generator are approximated with the forward differences method, a larger number of random directions $m$ allows one to compute more accurate gradients. However, in a model extraction setting, each additional gradient approximation step comes at the cost of increased query complexity. In practice, with a fixed query budget $Q$, changing the number of random directions directly impacts the proportion of queries used to train each network. 
This ratio of queries $r$ used to train the student is given by:
\begin{equation}
    r = \frac{n_S}{n_S + (m + 1)n_G} 
\end{equation}
In our setting (i.e. $n_G=1$, $n_S=5$) choosing $m$ equal to 1 or 10 respectively results in 71\% and 31\% of the query budget being used to directly train the student, while the remainder is spent to get better gradient estimates to train the generator. Despite using a  majority of queries to train the generator, the setting with $m=10$ achieves comparable accuracy.
In Table~\ref{table:m-values}, we observe how choosing lower values of $m$ achieves 85\% test  accuracy in much fewer queries.

\paragraph{Amortizing the cost.}
We hypothesize that since early in the training the discriminator (or student) provides only little signal, it is beneficial for the generator to initially rely on weak signals of gradient approximation. Effectively, this helps amortize the cost of gradient approximation over multiple epochs, and effectively pushes the expense to a later stage when the discriminator (or student) provides stronger signal. Figure~\ref{fig:amortizing} shows that relative to the query budget utilization, different values of $m$ perform similarly. %

This also suggests a hybrid strategy where the adversary first extracts a (somewhat poor) student model through distillation from a surrogate dataset. Indeed,  we show in Section~\ref{sec:surrogate} that surrogate datasets drawn from a different distribution (than the victim model's training data) enable distillation-based model extraction---albeit to a lower accuracy than in-distribution surrogate datasets. This poor initial student can then be improved by synthesizing queries with the data-free model extraction's generator to bring the student model closer to the victim model's performance.   

\paragraph{Case $\mathbf{m=1}$.} 
In the extreme case where the number of gradient approximation steps for forward differences ($m$) is set to 1, the approximated gradient is colinear to the random direction sampled for the approximation, but always points in the direction that helps maximize the loss (i.e. its projection onto the true gradient is positive). The cosine similarity with the true gradient is, thus, very small.
To validate this effect, we additionally experiment with $m=1$ where the approximate gradient was randomly flipped to the wrong direction with half probability. As hypothesized, the student accuracy did not improve beyond 20\% in our experiments on CIFAR10. 
This suggests that computing gradients that are extremely inaccurate  makes it possible to train the student as long as these gradients are  in the correct direction. %

\subsection{Impact of Logits Correction}
\label{subsec:exp-logits}

A model extraction attack should be applicable to the nature of predictions offered by MLaaS APIs. Most APIs provide per-class probability distribution rather than the true logits, since probabilities are more easily interpreted by the end user. To perform model extraction successfully we thus need to recover the logits from the victim's prediction probabilities. We show that is is possible to do so and recover approximate logits whose  Mean Average Error (MAE) with the true logits is low, on three different victim architecture.

The MAE reported in Table~\ref{table:logits-correction} are negligible compared to true logits which take values in the order of magnitude of 1. Therefore, the adversary can use these approximate logits in lieu of the true logits. 
In comparison, approximating true logits with plain log-probabilities resulted in a MAE in the order of magnitude of the true logits themselves. Using  the log-probabilities with such a large error makes the student training harder---it did not yield accuracy above 75\%.

This method is effective because the mean of the true logits is nearly 0 (see Table~\ref{table:logits-correction}). Therefore, subtracting the mean from the log-probabilities is equivalent to subtracting the additive constant $C(x)$ itself.

\section{Conclusions}
\label{sec:conclusions}

In this paper, we demonstrate that data-free model extraction is not only practical but also yields accurate copies of the victim model. This means that model extraction attacks is a credible threat to the intellectual property of models released intentionally or not to the public. We believe an interesting direction for future work is 
to detect such queries without decreasing the model's utility to legitimate users.

\section*{Acknowledgements}
We thank the reviewers and members of CleverHans Lab for their insightful feedback. This work was supported by a Canada CIFAR AI Chair, NSERC, a gift from  Microsoft, and sponsors of the Vector Institute.

{\small
\bibliographystyle{ieee_fullname}
\bibliography{egbib}
}
\clearpage
\appendix
\section{How Hard is it to Find a Surrogate Dataset?}
\label{app:surrogate}

\begin{table*}[t]
\centering
\begin{tabular}{@{}lr|rrrrrr@{}}
\toprule
        & Victim & CIFAR10 & CIFAR100 & SVHN   & MNIST  & SVHN$_{skew}$ & Random \\ \midrule
CIFAR10 & 95.5\% & 95.2\%  & 93.5\%   & 66.6\% & 37.2\% & -    & 10.0\%          \\
SVHN    & 96.2\% & 96.0\%  & -        & 96.3\% & 89.5\% & 96.1\%   & 84.1\%      \\ \bottomrule
\end{tabular}
\caption{Model Extraction accuracy across various surrogate datasets. Victim models were trained on the CIFAR10 and SVHN datasets, and the source accuracies are  reported under the heading `Victim'}
\label{table:surrogate-benchmark} 
\end{table*}

To motivate the need for data-free approaches to model extraction,  we show here that an adversary relying on a surrogate dataset must ensure that its distribution is close to the one of the victim's training set. Otherwise, model extraction will return a poor approximation of the victim.

Consider a \textit{victim} machine learning model, $\mathcal{V}$, trained on a proprietary dataset, $\mathcal{D_V}$. 
The victim model reveals its predictions through either a prediction API (as is common in MLaaS) or through the deployment of the model on devices accessible to adversaries.  
The adversary, $\mathcal{A}$,  attempts to steal $\mathcal{V}$ by querying it with a  surrogate dataset, $\mathcal{D_S}$. 
This surrogate dataset is assumed to be publicly available or easier to access because it does not need to be labeled.

We now perform a systematic study of the features that characterize the \textit{closeness} of $\mathcal{D_S}$ when compared to $\mathcal{D_V}$.
Let the private and surrogate datasets \(\mathcal{D_V}\) and \(\mathcal{D_S}\) be characterized by \(\mathcal{D} = \{\mathcal{X},P(X),\mathcal{Y}, P(Y|X)\}\)~\cite{ruder2017transferlearning}. The private and surrogate datasets can vary in three ways, which we will illustrate in the following with object recognition tasks: 

\begin{enumerate}
    \item \textbf{A: }$(\mathcal{X_V} \neq \mathcal{X_S})$. 
    When the inputs of \(\mathcal{D_V}\) and \(\mathcal{D_S}\) belong to different feature spaces (i.e. domain). In computer vision, for example, this can be a scenario wherein the input data (e.g., images, videos) for the datasets contain a different number of channels or pixels.

    \item \textbf{B: }$(P(X_\mathcal{V}) \neq P(X_\mathcal{S}))$. 
    While the input domain of both the surrogate and private datasets is the same, their marginal probability distribution is different. For example,  when the semantic nature is different for the two datasets (images of animals, digits, etc).
    
    \item \textbf{C: }$(P(Y_\mathcal{V}|X_\mathcal{V}) \neq P(Y_\mathcal{S}|X_\mathcal{S}))$. 
    We consider a setting where the semantic distribution $(P(X))$ is the same, but the class-conditional probability distributions of the victim and surrogate training sets are different, e.g. when the two datasets have class imbalance.
\end{enumerate}

\subsection{Optimization Problem}
The logit distribution of the victim network can often have a strong affinity toward the true label. To address this issue, Hinton et al.  suggested scaling the logits to make the probability distributions more informative~\cite{hinton2015distilling}. 
\begin{align*}
\mathcal{V}_i(x) = \frac{\exp{\left(v_i(x)/\tau\right)}}{\sum_j \exp{\left(v_j(x)/\tau\right)}}; \quad 
\mathcal{L} = - \tau^2
\text{KL}\left(
\mathcal{V}(x), \mathcal{S}(x)
\right)
\end{align*}
where $\tau>1$ is referred to as the temperature scaling parameter. In the knowledge distillation literature, a combination of both the cross entropy and knowledge distillation loss are used to query the teacher~\cite{cho2019efficacy}. However, model extraction we rely only on the KL divergence loss because the queries are made from a surrogate dataset which may not have any semantic binding to the true class.

\subsection{Experimental Setting}

\paragraph{Hyperparameters} 
To search over a meaningful hyperparameter space for the temperature co-efficient $\tau$, we refer to prior work in knowledge distillation such as~\cite{cho2019efficacy,hinton2015distilling,lan2018knowledge} to confine the search over $\tau \in \{1, 3, 5, 10\}$.  We used the SGD optimizer, and train the CIFAR10 students for 100 epochs, while the SVHN students were trained for 50 epochs. We experimented with two different learning schedules: (1) cyclic learning rate~\cite{smith2017cyclical}; and (2) step-decay learning rate. For step-decay, we reduced the learning rate by a factor of 0.2 at 30\%, 60\%, and 80\% of the training process. In both cases, the maximum learning rate was set to 0.1.

\paragraph{Experimental validation.} 
To illustrate our argument, we next detail the relation between a task-specific surrogate dataset and the accuracy of state-of-the-art model extraction techniques. The victim models under attack are ResNet-18-8x models, their accuracy is reported in Table~\ref{table:surrogate-benchmark}. Further details on the victim models training are provided in Section~\ref{sec:exp}). We find that querying from the original dataset yields the most query-efficient and accurate extraction results. This is not surprising given that this setup corresponds to the original knowledge distillation setting. Our observations are made on both CIFAR10 and SVHN:
\begin{itemize}
    \item \textbf{CIFAR10.} We benchmarked model extraction on 4 surrogate datasets, each reflecting a different property detailed above: CIFAR10~\cite{krizhevsky2009learning}, CIFAR100~\cite{krizhevsky2009learning} (BC), SVHN~\cite{netzer2011reading} (AB) and MNIST~\cite{lecun1998gradient} (AB). To ensure a fair comparison, we  bound the maximum number of distinct samples queried by 50,000 while performing model extraction.
    \item \textbf{SVHN.} We evaluated model extraction by querying from SVHN, SVHN$_{skew}$ (C), MNIST (A) and CIFAR10 (B) as surrogate datasets. Similar to the case for CIFAR10, we cap the maximum number of distinct samples queried to 50,000.
\end{itemize}

Finally, as a control for our experiments, we also studied the extraction accuracy of the models when trained using totally random queries.

\paragraph{Dataset adaptation.} The SVHN, CIFAR10, and CIFAR100 datasets contain $32\times32$ color images. To query networks trained on CIFAR10 and SVHN with images from the MNIST dataset, which contains grayscale images of size $28\times28$, we re-scaled the image and repeated the same input across all three RGB channels. In case of random input generation, we sample input tensors from a normal distribution with mean 0 and variance 1. Note that the teacher networks were trained on normalized datasets in the first place. Finally, in case of SVHN$_{skew}$, we  supplied images from only the first 5 classes of the dataset to skew the distribution of the modified dataset.

\paragraph{Results.}
We present the results for accuracy of extracted models across various surrogate datasets for CIFAR10 and SVHN in Table~\ref{table:surrogate-benchmark}.
Recall that both the CIFAR10 and CIFAR100 datasets are subsets from the same TinyImages~\cite{torralba2008tinyimages} dataset. We find that the identical source distribution was extremely useful in making relevant queries to the CIFAR10 teacher. The accuracy of the extracted model reached 93.5\%, just below the 95.5\% accuracy of the teacher model.
However, when we used the SVHN surrogate dataset to query the CIFAR10 teacher, with a different source distribution, the model extraction performance dropped remarkably, attaining a maximum of 66.6\% across all of the hyperparameters tried. 
In the most extreme scenario when querying the CIFAR10 teacher with MNIST--a dataset with disjoint feature space both in terms of number of pixels, and number of channels)---model accuracy did not improve beyond 37.2\%.

On the contrary, we notice that the victim trained on the SVHN dataset is much easier for the adversary to extract. Surprisingly, even when the victim is queried with completely random inputs, the extracted model attains an accuracy of over 84\% on the original SVHN test set. Further, nearly all surrogate datasets are able to achieve greater than 90\% accuracy on the test set. We hypothesize that this observation is linked to how the digit classification task, at the root of SVHN, is a simpler task for neural networks to solve, and the underlying representations (hence, not being as complex as for CIFAR10) can be learnt even when queried over random inputs.

Given the current understanding of model extraction, 
we make two conclusions: (1) the success of model extraction largely depends on the complexity of the task that the victim model aims to solve; and (2) similarity to source domain is critical for extracting machine learning models that solve complex tasks.
We posit that it may be nearly as expensive for the adversary to extract a CIFAR10 machine learning model with a good surrogate dataset, as is training from scratch. A weaker or non-task specific dataset may have lesser costs, but has high accuracy trade-offs.

\section{Recovering logits from probabilities}
\label{app:logits-correction}
The main difficulty with computing $\mathcal{L}_{\ell_1}$ is that it requires access to $\mathcal{V}$'s logits $v_i$, but we only have access to the probabilities of each class (i.e., after the softmax is applied to the logits). In a first approximation, the logits can be recovered by computing the log-probabilities but the resulting approximate logits are computed up to an additive constant $C(x)$ to which we don't have access in a black-box setting. This additive constant is the same for all logits but is different from one image to another. Related works on adversarial examples~\cite{bhagoji2018practical, chen2017zoo} use losses that are the difference of two logits, effectively canceling out the additive constant. In our case, the logits need to be used individually which makes the $\ell_1$ loss more difficult to compute in our setting.

To overcome this issue, we propose to approximate the true logits of each image $x$ in two steps. First, compute the logarithm of the probability vector $V(x)$. 
\begin{alignat}{2}
\Tilde{v}_i(x) &= \log \mathcal{V}_i(x)
    &= v_i + C(x)
\end{alignat}
Then, compute the approximate true logits $v^*_i(x)$ by subtracting the log-probability vector with its own mean: 
\begin{alignat}{2}
    v^*_i (x) = \Tilde{v}_i(x) - \frac{1}{K}\sum_{j=1}^K \Tilde{v}_j(x)
    \notag\\
    = v_i(x) - \frac{1}{K}\sum_{j=1}^K v_j(x) &\approx  v_i(x)
\end{alignat} 
The second equality holds because the mean of the log-probability vector $\Tilde{v}_i(x)$ is equal to the mean of the true victim logits $v_i(x)$ plus the mean of the additive constant (i.e. the $C(x)$ itself). By analyzing the mean values of the true logits from various pre-trained models---which proves to be negligible in comparison to the logit values themselves, we provide empirical evidence in Section~\ref{subsec:exp-logits} that this recovers a highly accurate  approximation of the true logits $v^*_i (x)$. 

\section{Examples of Synthetic Images}

Figure~\ref{fig:synthetic} shows 4 images from the generator towards the end of the attack on CIFAR-10. We do not observe any similarities with the images from the original training dataset.

\begin{figure}[h]
\centering
	\includegraphics[height=0.9\columnwidth]{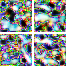}
	\caption{Four synthetic images from the generator.} 
	\label{fig:synthetic}
\end{figure}

\section{Hypothesis 1: Justification}
\label{app:proof-th}

\subsection{Preliminary results}
\begin{lemma}
\label{lemma:eigenvalues}
If $\mathcal{S}(x) \in (0, 1)^K$ is the softmax output of a differentiable function (e.g. a neural network) on an input $x$ and $s$ is the corresponding logits vector, then the Jacobian matrix $J = \pder[\mathcal{S}]{s}$ has an eigenvalue decomposition and all its eigenvalues are in the interval [0, 1].
\end{lemma}

\begin{proof}
By definition:
\begin{equation*}
    \forall i \in \{1 \dots K\}, ~\mathcal{S}_i = \frac{\exp(s_i)}{\sum_{k=1}^K \exp(s_k)}
\end{equation*}

For some $i, j \in \{1 \dots K\}$, if $i\neq j$:

\begin{align*}
    \pder[\mathcal{S}_i]{s_j} &= - \exp(s_j) \frac{\exp(s_i)}{{(\sum_{k=1}^K \exp(s_k))}^2} \\
    &= - \mathcal{S}_i \mathcal{S}_j
\end{align*}
    
if $i = j$:
\begin{align*}
\pder[\mathcal{S}_i]{s_j} &= \frac{ \exp(s_i) (\sum_{k=1}^K \exp(s_k)) - \exp(s_j) \exp(s_i)}{{(\sum_{k=1}^K \exp(s_k))}^2}\\
&= \frac{ \exp(s_i)}{\sum_{k=1}^K \exp(s_k)} - \frac{ \exp(s_i)^2}{(\sum_{k=1}^K \exp(s_k))^2} \\
&= \mathcal{S}_i (1 - \mathcal{S}_i)
\end{align*}

Therefore, $\forall x,$
\begin{equation*}
 J = \pder[\mathcal{S}]{s} =
\begin{bmatrix}
\mathcal{S}_1 (1 - \mathcal{S}_1) & -\mathcal{S}_1 \mathcal{S}_2 & \dots & -\mathcal{S}_1 \mathcal{S}_K \\ 
-\mathcal{S}_1 \mathcal{S}_2 & \mathcal{S}_2 (1 - \mathcal{S}_2) & \dots & -\mathcal{S}_2 \mathcal{S}_K \\
\vdots  & \vdots &  & \vdots \\ 
-\mathcal{S}_1 \mathcal{S}_K & -\mathcal{S}_2 \mathcal{S}_K & \dots & \mathcal{S}_K (1 - \mathcal{S}_K)\\ 
\end{bmatrix}
\end{equation*}

The matrix $J$ is real-valued symmetric, therefore it has an eigen-decomposition with real eigenvalues. $\exists \lambda_1,\lambda_2, \dots, \lambda_K \in \mathbb{R}, X_1, X_2, \dots, X_K \neq 0$ such that:

\begin{equation*}
    \forall i \in \{1 \dots K \}, J X_i = \lambda_i X_i
\end{equation*}

Let us prove that all eigenvalues are in the interval [0, 1]. Suppose for a contradiction that one eigenvalue $\lambda$ is strictly negative. Let the associated eigenvector be: 
\begin{equation*}
X = [x_1, x_2, \dots, x_K]^T
\end{equation*}

The $i$-th component of the vector $JX$ is:
\begin{align*}
[JX]_i &= \mathcal{S}_i x_i - \mathcal{S}_i \sum_{k=1}^K x_k \mathcal{S}_k\\
&= \mathcal{S}_i x_i - \mathcal{S}_i \langle X , \mathcal{S}\rangle
\end{align*}

where $\langle\cdot,\cdot\rangle$ is the standard inner product.

Since $X$ is an eigenvector we have,

\begin{equation*}
    JX = \lambda X
\end{equation*}

So $\forall~ i$,
\begin{align*}
    \mathcal{S}_i x_i - \mathcal{S}_i \langle X , \mathcal{S}\rangle &= \lambda x_i \\
    x_i (\mathcal{S}_i - \lambda) &= \mathcal{S}_i \langle X , \mathcal{S}\rangle
\end{align*}

Since $X \neq 0$, $\exists~i_0$ such that $x_{i_0} \neq 0$. Furthermore, $\lambda$ is strictly negative so:

\begin{equation*}
    x_{i_0} (\mathcal{S}_{i_0} - \lambda) = \mathcal{S}_{i_0} \langle X , \mathcal{S}\rangle \neq 0
\end{equation*}

Therefore, the inner product on the right hand side is non-zero.

In addition, $\lambda<0$ implies that $\mathcal{S}_i - \lambda > \mathcal{S}_i > 0$ so $ \forall~i, x_i$ and $\langle X , \mathcal{S}\rangle$ have the same sign. There are two cases left. 

\underline{If $\langle X , \mathcal{S}\rangle > 0$}, then $\forall~i, x_i > 0$ and:

\begin{align*}
     x_i (\mathcal{S}_i - \lambda) &> x_i \mathcal{S}_i\\
     \mathcal{S}_i \langle X , \mathcal{S}\rangle &> x_i \mathcal{S}_i
\end{align*}

By summing on all $i$ we obtain:

\begin{align*}
\sum_{i=1}^K \mathcal{S}_i \langle X , \mathcal{S}\rangle &> \sum_{i=1}^K x_i \mathcal{S}_i\\
\langle X , \mathcal{S}\rangle \sum_{i=1}^K \mathcal{S}_i &> \langle X , \mathcal{S}\rangle\\
\langle X , \mathcal{S}\rangle &> \langle X , \mathcal{S}\rangle
\end{align*}

Which is an \textbf{absurdity}.
\newline

\underline{If $\langle X , \mathcal{S}\rangle < 0$}, then $\forall~i, x_i < 0$ and:

\begin{align*}
     x_i (\mathcal{S}_i - \lambda) &< x_i \mathcal{S}_i\\
     \mathcal{S}_i \langle X , \mathcal{S}\rangle &< x_i \mathcal{S}_i
\end{align*}

The same summation and reasoning yields an \textbf{absurdity}. We just proved that all the eigenvalues of $J$ are non-negative.

Lastly, the trace of the Jacobian matrix  $tr(J)$ equals the sum of all eigenvalues. Computing the trace yields:

\begin{align*}
    tr(J) = \sum_{i=1}^K \lambda_i &= \sum_{i=1}^K \mathcal{S}_i (1 - \mathcal{S}_i)\\
    &= \sum_{i=1}^K \mathcal{S}_i - \sum_{i=1}^K \mathcal{S}_i^2\\
    &= 1 - \sum_{i=1}^K \mathcal{S}_i^2 < 1\\
\end{align*}

Since $\lambda_i \geq 0$ and $\sum_{i=1}^K \lambda_i < 1$, all eigenvalues must be in the interval [0, 1], which concludes the proof.

\end{proof}

\begin{lemma}
\label{lemma:norm-product}
In the same setting as for Lemma~\ref{lemma:eigenvalues}, if $J$ is the Jacobian matrix $ \pder[\mathcal{S}]{s}$ then for any vector $Z$ we have:

\begin{equation*}
    \|JZ\| \leq \|Z\|
\end{equation*}
\end{lemma}

\begin{proof}
Let $\lambda_1, \lambda_2, \dots, \lambda_K$ be the eigenvalues of $J$ and $X_1, X_2, \dots, X_K$ be the associated eigenvectors. We can decompose $Z$ with the orthonormal eigenvector basis:

\begin{equation*}
    Z = \sum_{i=1}^K \alpha_i X_i
\end{equation*}

Computing the product $JZ$ yields:

\begin{align*}
    JZ = \sum_{i=1}^K \lambda_i \alpha_i X_i 
\end{align*}

The norm of the product is:
\begin{align*}
    \|JZ\| = (JZ)^T (JZ) &= \sum_{i=1}^K \lambda_i^2 \alpha_i^2\\
    &\leq \sum_{i=1}^K \alpha_i^2
\end{align*}
because $\forall~i, |\lambda_i| \leq 1$ (see Lemma~\ref{lemma:eigenvalues}). Since the eigenvector basis is orthonormal we have 

\begin{equation*}
 \|JZ\| \leq \sum_{i=1}^K \alpha_i^2 = \| Z \|
\end{equation*}
\end{proof}

\begin{lemma}
\label{lemma:similarity-jacobians}
Let $\mathcal{S}(x)$ and $\mathcal{V}(x)$ be the softmax output of two differentiable functions (e.g. neural networks) on an input $x$, with respective logits $s(x)$ and $v(x)$. When $\mathcal{S}$ converges to $\mathcal{V}$, then $\pder[\mathcal{S}]{s}$ converges to $\pder[\mathcal{V}]{v}$.
\end{lemma}

\begin{proof}
Recall that
\begin{equation*}
\pder[\mathcal{S}]{s} =
\begin{bmatrix}
\mathcal{S}_1 (1 - \mathcal{S}_1) & -\mathcal{S}_1 \mathcal{S}_2 & \dots & -\mathcal{S}_1 \mathcal{S}_K \\ 
-\mathcal{S}_1 \mathcal{S}_2 & \mathcal{S}_2 (1 - \mathcal{S}_2) & \dots & -\mathcal{S}_2 \mathcal{S}_K \\
\vdots  & \vdots &  & \vdots \\ 
-\mathcal{S}_1 \mathcal{S}_K & -\mathcal{S}_2 \mathcal{S}_K & \dots & \mathcal{S}_K (1 - \mathcal{S}_K)\\ 
\end{bmatrix}
\end{equation*}

If $\mathcal{V}_i(x) - \mathcal{S}_i(x) = \epsilon_i(x)$, then:

\begin{align*}
    \mathcal{S}_i (1 - \mathcal{S}_i) &= (\mathcal{V}_i + \epsilon_i) (1 - \mathcal{V}_i -\epsilon_i)\\
    &= \mathcal{V}_i (1 - \mathcal{V}_i) + \epsilon_i (1 - \mathcal{V}_i) - \epsilon_i^2\\
    &= \mathcal{V}_i (1 - \mathcal{V}_i) + o(1)
\end{align*}

and

\begin{align*}
    -\mathcal{S}_i \mathcal{S}_j &= -(\mathcal{V}_i + \epsilon_i) (\mathcal{V}_j + \epsilon_j)\\
    &= -\mathcal{V}_i \mathcal{V}_j - \mathcal{V}_i \epsilon_j -\mathcal{V}_j \epsilon_i - \epsilon_i \epsilon_j\\
    &= -\mathcal{V}_i \mathcal{V}_j + o(1)
\end{align*}

Therefore, we can write:

\begin{equation*}
    \pder[\mathcal{S}]{s} = \pder[\mathcal{V}]{v} + \Bar{\epsilon}(x)
\end{equation*}

where $\Bar{\epsilon}(x)$ converges to the null matrix as $\mathcal{S}$ converges to $\mathcal{V}$. In other words we can write:

\begin{align*}
    \pder[\mathcal{S}]{s} \underset{\mathcal{S} \rightarrow \mathcal{V}}{\approx} \pder[\mathcal{V}]{v}
\end{align*}

\end{proof}

\subsection{Justification of the hypothesis.} %

Hypothesis~\ref{th:grad-vanish} states that for two differentiable functions with softmax output $\mathcal{S}$ and $\mathcal{V}$, and respective logits $s$ and $v$, the gradients of the KL divergence loss $\mathcal{L}_{KL}$ with respect to the input should be small compared to the gradients of the $\ell_1$ norm loss $\mathcal{L}_{\ell_1}$ as $\mathcal{S}$ converges to $\mathcal{V}$. $\forall x \in [-1, 1]^d$:

\begin{equation*}
    \|\nabla_x \mathcal{L}_{\text{KL}}(x)\| \underset{\mathcal{S} \rightarrow \mathcal{V}}{\ll} \|\nabla_x \mathcal{L}_{\ell_1}(x)\| 
\end{equation*}

\begin{proof}

First, we note that:

\begin{equation*}
    \sum_{i=1}^K \mathcal{S}_i = 1
\end{equation*}

implies

\begin{equation*}
    \sum_{i=1}^K \pder[\mathcal{S}_i]{x} = \Vec{0}
\end{equation*}

And the same holds for $\mathcal{V}$ because both are probability distributions.

Then we compute the gradients for both loss functions:

\underline{For the {$\ell_1$ norm loss}:}

\begin{equation*}
\nabla_x \mathcal{L}_{\ell_1}(x) = \sum_{i=1}^K {sign}(v_i - s_i) \left( \pder[v_i]{x} - \pder[s_i]{x} \right) 
\end{equation*}

\underline{For the {KL divergence loss:}}

\begin{align*}
\nabla_x \mathcal{L}_{\text{KL}}(x) &= \sum_{i=1}^K \pder[\mathcal{V}_i]{x} \log \mathcal{V}_i + 1 \pder[\mathcal{V}_i]{x} - \pder[\mathcal{V}_i]{x} \log \mathcal{S}_i - \pder[\mathcal{S}_i]{x} \frac{\mathcal{V}_i}{\mathcal{S}_i}\\
&= \sum_{i=1}^K \pder[\mathcal{V}_i]{x} + \sum_{i=1}^K \pder[\mathcal{V}_i]{x} \log \frac{\mathcal{V}_i}{\mathcal{S}_i} - \pder[\mathcal{S}_i]{x} \frac{\mathcal{V}_i}{\mathcal{S}_i}\\
&= \sum_{i=1}^K \pder[\mathcal{V}_i]{x} \log \frac{\mathcal{V}_i}{\mathcal{S}_i} - \pder[\mathcal{S}_i]{x} \frac{\mathcal{V}_i}{\mathcal{S}_i}\\
\end{align*}

When $\mathcal{S}$ converges to $\mathcal{V}$, we can write

\begin{equation*}
    \mathcal{V}_i(x) = \mathcal{S}_i(x) (1 + \delta_i(x))
\end{equation*}

where $\delta_i(x) \underset{\mathcal{S} \rightarrow \mathcal{V}}{\rightarrow} 0$. 

Since $\log 1 + x \approx x$ when $x$ is close to $0$ we can write:

\begin{align*}
\nabla_x \mathcal{L}_{\text{KL}}(x) &= \sum_{i=1}^K \pder[\mathcal{V}_i]{x} \log \frac{\mathcal{V}_i}{\mathcal{S}_i} - \pder[\mathcal{S}_i]{x} \frac{\mathcal{V}_i}{\mathcal{S}_i}\\
&\approx\sum_{i=1}^K \pder[\mathcal{V}_i]{x} \delta_i - \pder[\mathcal{S}_i]{x} (1 + \delta_i)\\
&\approx\sum_{i=1}^K \delta_i \left(\pder[\mathcal{V}_i]{x} - \pder[\mathcal{S}_i]{x}\right) + \sum_{i=1}^K \pder[\mathcal{S}_i]{x}\\
&\approx\sum_{i=1}^K \delta_i \left(\pder[\mathcal{V}_i]{x} - \pder[\mathcal{S}_i]{x}\right)\\
&\approx\sum_{i=1}^K \delta_i \pder[\mathcal{V}]{v} \left(\pder[v_i]{x} - \pder[s_i]{x}\right) \text{(Lemma~\ref{lemma:similarity-jacobians})}\\
&\approx \pder[\mathcal{V}]{v} \sum_{i=1}^K \delta_i  \left(\pder[v_i]{x} - \pder[s_i]{x}\right)\\
\end{align*}

Using Lemma~\ref{lemma:norm-product}, the norm is upper bounded by:

\begin{equation}
\label{eq:norm-kl-bound}
    \|\nabla_x \mathcal{L}_{\text{KL}}(x)\| \leq \norm{\sum_{i=1}^K \delta_i  \left(\pder[v_i]{x} - \pder[s_i]{x}\right)}
\end{equation}

For the $\ell_1$ norm, however, the norm is:

\begin{equation}
\label{eq:norm-l1-bound}
\|\nabla_x \mathcal{L}_{\ell_1}(x)\| = \norm{\sum_{i=1}^K {sign}(v_i - s_i) \left( \pder[v_i]{x} - \pder[s_i]{x} \right) }
\end{equation}

From equation~\ref{eq:norm-kl-bound}, we can observe that each term is negligible compared to its counterpart in equation~\ref{eq:norm-l1-bound}: for all $i$ we have:

\begin{align*}
\norm{ \delta_i  \left(\pder[v_i]{x} - \pder[s_i]{x}\right)}
&\leq \epsilon \norm{ \left(\pder[v_i]{x} - \pder[s_i]{x}\right)}\\
\end{align*}

And also $\forall i$:

\begin{equation*}
\norm{{sign}(v_i - s_i) \left( \pder[v_i]{x} - \pder[s_i]{x} \right) } = \norm{\left( \pder[v_i]{x} - \pder[s_i]{x} \right) }
\end{equation*}

Therefore, by summing these terms on the index $i$ we can expect the KL divergence gradient to be small in magnitude compared to those of the $\ell_1$ norm. However, it does not seem possible to prove this result rigorously without further assumptions on the data distribution or the mode of convergence of $\mathcal{S}$.

\end{proof}

\end{document}